\documentclass{article}
\usepackage{spconf,amsmath,graphicx,amssymb,bm,amsthm,subcaption,dsfont}
\usepackage{algorithmic,algorithm}
\usepackage{balance,cite,color}

\newcommand{\col}[1]{\mathrm{col}\big\{#1\big\}}
\newcommand{\grad}[1]{\nabla_{#1^{\tran}} }
\newcommand{\w}{\bm{w}}

\newcommand{\we}{\widetilde{\w}}
\newcommand{\eqdef}{\:\overset{\Delta}{=}\:}
\DeclareMathOperator*{\argmin}{argmin}

\newcommand{\tran}{{\sf T}}

\newtheorem{theorem}{Theorem}
\newtheorem{assumption}{Assumption}
\newtheorem{lemma}{Lemma}

\title{Asynchronous Diffusion Learning with Agent Subsampling
	and Local Updates}
\name{Elsa Rizk\thanks{School of Engineering, École Polytechnique Fédérale de Lausanne (e-mail:\{elsa.rizk, ali.sayed\}@epfl.ch).}, Kun Yuan\thanks{Center for Machine Learning Research, Peking University (e-mail: kunyuan@pku.edu.cn).}, Ali H. Sayed}
\address{}

\begin{document}
\ninept
\maketitle
\begin{abstract}
In this work, we examine a network of agents operating asynchronously, aiming to discover an ideal global model that suits individual local datasets. Our  assumption is that each agent independently chooses when to  participate throughout the algorithm and the specific subset of its neighbourhood with which it will cooperate at any given moment. When an agent chooses to take part, it undergoes multiple local updates before conveying its outcomes to the sub-sampled neighbourhood. Under this setup, we prove that the resulting asynchronous diffusion strategy is stable in the mean-square error sense and provide performance guarantees specifically for the federated learning setting. We illustrate the findings with numerical simulations.

\end{abstract}
\section{Introduction and Related Material}
\label{sec:intro}
Networks of agents comprise individual agents collaborating in the pursuit of a common goal. In the case of distributed optimization, the agents aim to solve a global optimization problem by utilizing local information. Typically, the available solutions involve a local update step followed by an aggregation step among neighbours. Examples of such strategies include incremental methods \cite{Bert1997-inc,Catt2011-inc,lopes2007incremental,Elias09-inc,Nedic01-inc}, consensus methods \cite{DeGroot,Nedic09,xiao2004fast,Johansson2008SubgradientMA}, and diffusion methods \cite{Lopes08,chen2012diffusion,Tu12,chen2015learning,Vlaski21}. Commonly, these methods assume full agent participation in every iteration and a one-to-one ratio of aggregation steps to local update steps. Nonetheless, various situations like agent drop-outs do not satisfy these assumptions \cite{6483403,hadjicostis2013average}, where certain agents might not partake in each iteration. Furthermore, computational and communication limitations \cite{tsianos2012communication,10097010} could necessitate an agent to perform several local updates before sharing outcomes with a subset of its neighbours, as opposed to all of them.

Therefore, this work focuses on asynchronous networks of agents and modifies the algorithms, particularly diffusion-type algorithms, to allow flexibility in agent participation, agent sub-sampling, and local updates. Relevant research on asynchronous distributed learning can be found in \cite{1104412,srivastava2011distributed,kar2009distributed2,kar2009distributed,pmlr-v139-aviv21a,9036910,pmlr-v80-lian18a, zhao2014asynchronous,lopes2010randomized,5757817,arablouei2015analysis}. The primary relevant prior work \cite{zhao2014asynchronous} in the context of asynchronous networks operates under the assumption of independent step-size and combination weights as well as a common optimal model. However, within the scenarios we are investigating, such as federated learning \cite{mcmahan16}, such assumptions do not hold. As a result, we abandon the independence assumption and study the stability and performance of a particular asynchronous distributed setup. Our framework encompasses time-varying network topologies that allow local updates. Therefore, it appears this work would be the first to establish an explicit mean-square deviation (MSD) expression for the federated learning scenario, as well as similar asynchronous learning algorithms that incorporates local updates. 

In the upcoming sections, we describe the asynchronous network and illustrate how the original federated learning algorithms initially presented in \cite{mcmahan16} can be understood as a specific example within this broader framework. Moving forward, the second section provides evidence of the stability of the asynchronous adapt-then-combine (ATC) diffusion algorithm in the federated learning setting. In the subsequent third section, we conduct a comprehensive performance analysis, culminating in the derivation of an MSD expression for the particular case of federated learning. Lastly, in the fourth section, we execute a series of experiments to further explore these concepts.

\section{Asynchronous Network}\label{sec:}
\subsection{Problem Setup}
We consider a network of $K$ agents, similar to the federated or fully decentralized setting, all aiming to solve the convex optimization problem presented as follows:
\begin{align}\label{eq:optProb}
	\min_w  \frac{1}{K}\sum_{k=1}^K  J_k(w),
\end{align}
where the local risk function $J_k(\cdot)$ is defined as an empirical average of the local loss function $Q_k(\cdot;x_{k,n})$ over the local dataset $\{x_{k,n}\}_{n=1}^{N_k}$. The communication among agents is restricted by an underlying graph structure, where the combination matrix denoted by $A$ holds elements $a_{\ell k}$ representing the weight agent $k$ assigns to information shared by agent $\ell$. We impose certain assumptions on the structure of the underlying graph and the nature of the risk and loss functions.

\begin{assumption}[\textbf{Combination matrix}]\label{assum:mat}
	The combination matrix is left-stochastic, namely $\mathds{1}^\tran A = \mathds{1}$ with $a_{k\ell} \geq 0$. 
	\qed
\end{assumption}

\begin{assumption}[\textbf{Risk and loss functions}]\label{assum:fcts}
	The empirical risks $J_{k}(\cdot)$ are $\nu-$strongly convex. The loss functions $Q_{k}(\cdot;\cdot)$ are convex and twice differentiable, namely, for some $\nu > 0$:
	\begin{align} \label{eq:assFctConv}
		&J_{k}(w_2) \geq  \: J_{k}(w_1) + \grad{w}J_{k}(w_1)(w_2-w_1) + \frac{\nu}{2}\Vert w_2 - w_1 \Vert^2,\\
		&Q_{k}(w_2;\cdot) \geq  \: Q_{k}(w_1;\cdot) + \grad{w}Q_{k}(w_1;\cdot) (w_2 - w_1).
	\end{align}
	Furthermore, the loss functions have $\delta-$Lipschitz continuous gradients:
	\begin{equation}\label{eq:assFctLip}
		\Vert \grad{w}Q_{k}(w_2;x_{k,n}) - \grad{w}Q_{k}(w_1;x_{k,n})\Vert \leq \delta \Vert w_2 - w_1\Vert.
	\end{equation}
	\qed
\end{assumption}
\begin{assumption}[\textbf{Bounded and smooth Hessians}]\label{assum:Hess}
	The Hessians have bounded eigenvalues:
	\begin{align}
		\lambda_{\min} \leq \lambda \left( \grad{w}^2J_k(w) \right)\leq  \lambda_{\max},
	\end{align}
	and are locally Lipschitz  in a small neighbourhood around $w^o$ later defined in \eqref{eq:optMod}, namely, there exists $\kappa > 0$ for small $\Delta w$:
	\begin{align}
		\Vert \grad{w}^2 J_k(w^o + \Delta w) - \grad{w}^2J_k(w^o)\Vert \leq \kappa \Vert \Delta w\Vert.
	\end{align}
	\qed
\end{assumption} 

\begin{figure}[h!]
	\centering
	\begin{subfigure}{0.23\textwidth}
		\centering
		\includegraphics[width=\textwidth]{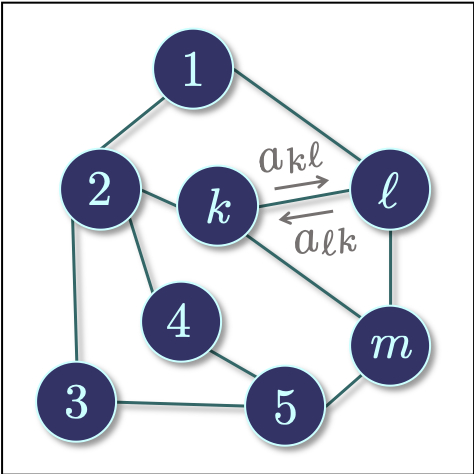}
		\caption{Underlying network.}\label{fig:neta}
	\end{subfigure}
	\begin{subfigure}{0.24\textwidth}
		\centering
		\includegraphics[width=\textwidth]{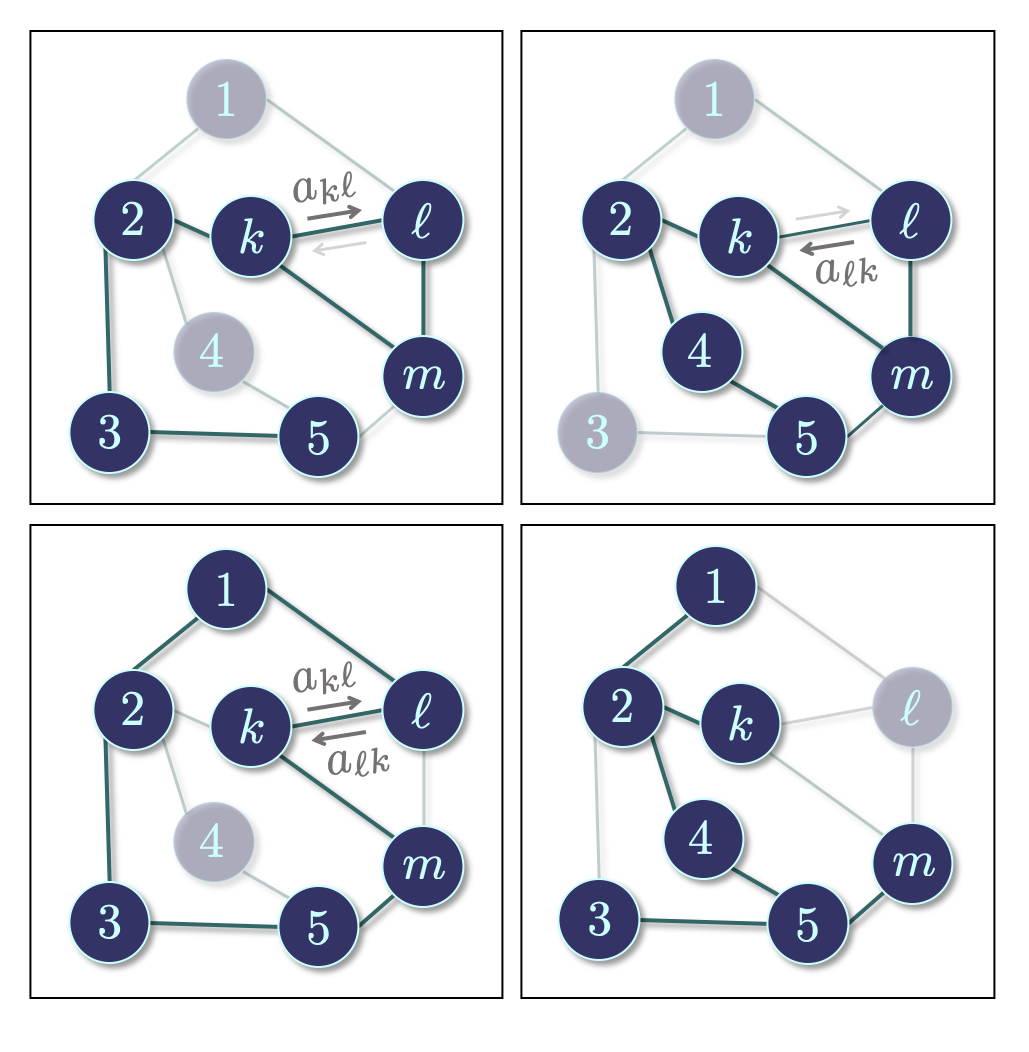}
		\caption{Time varying network.}\label{fig:netb}
		\label{fig:neti}
	\end{subfigure}
	\caption{Illustration of an asynchronous network whose nodes and links change with time.}\label{fig:net}
\end{figure}

We introduce the following assumptions on the agents' mode of operation. During an iteration $i$ of the algorithm, an agent $k$ has the option to engage. In the event of participation, it may sample a subset of its neighbourhood $\mathcal{N}_k$, from which it will aggregate their messages. We attribute a probability $q_k$ to the participation of agent $k$ and we let $q_{\ell k}$ be the sampling probability of agent $\ell$ by $k$. Furthermore, prior to any combination step, agent $k$ runs a total of $T$ local update steps. This setup can be modeled as an asynchronous network, where the combination matrix is time varying and random. Accordingly, we let $t \in {1,2,\cdots,T}$ denote the local iterations while $i$ denotes the global iterations. The time varying and random combination matrix is written as $\bm{A}_{(i-1)T+t}$. Throughout the local iterations, i.e., for  $t\neq T$, the combination matrix simplifies to the identity matrix, i.e., $\bm{A}_{(i-1)T+t} = I$. Yet, once $t=T$, a combination step follows after the $T$th update step. As a result, the combination matrix $\bm{A}_{iT}$ will be a sampled version of the original combination matrix $A$. For example, in Fig. \ref{fig:netb} top left, agent $m$ chooses to participate and it samples agents $\ell$ and $k$ while leaving out agent $5$. Therefore:
\begin{align}
	\bm{a}_{\ell m ,iT} &= a_{\ell m},  \quad	\bm{a}_{ k m, iT} = a_{k m }, \quad \bm{a}_{ 5 m, iT} = 0, \notag \\
	\bm{a}_{mm, iT} &= 1 - a_{\ell m}-a_{ km}.
\end{align}
Since agent $1$ is not participating, then all the weights it attributes to its neighbours will be 0 and its self-weight will be 1, i.e.:
\begin{align}
	\bm{a}_{11,iT} = 1, \quad \bm{a}_{2 1,iT} =\bm{a}_{\ell 1,iT} = 0.
\end{align}
 As such, at $t=T$, the elements of the matrix will be given by:
\begin{align}
	\bm{a}_{ \ell k, iT} = \begin{cases}
		a_{ \ell k}, & \text{with probability } q_k q_{\ell k} \\
		1 - \sum\limits_{m \in \mathcal{N}_k} \bm{a}_{ mk,iT},  & k = \ell \text{ with probability } q_k  \\
		1, & k= \ell \text{ with probabiloty } 1-q_k \\
		0, & \text{otherwise}
	\end{cases}
\end{align}
To ensure the matrix remains left-stochastic, each agent modifies its self-weight based on the neighbours it has sampled. Moreover, the step-size is also time varying and random:
 \begin{align}\label{eq:stepSize}
 	\bm{\mu}_{k,(i-1)T+t} = \begin{cases}
 		\mu, &  \text{with probability } q_k \\
 		0, & \text{otherwise}
 	\end{cases}
 \end{align}
 If we consider the ATC diffusion algorithm, the asynchronous version of it can thus be described as follows:
 \begin{align}
 	\bm{\psi}_{k,(i-1)T+t} &= \w_{k,(i-1)T+t-1} \notag \\
 	&\quad -
 	 \bm{\mu}_{k,(i-1)T+t} \widehat{\grad{w}J_{k}}(\w_{k,(i-1)T+t-1}), \\
 	\w_{k,(i-1)T+t} &= \sum_{\ell \in \mathcal{N}_k} \bm{a}_{ \ell k, (i-1)T+t}  	\bm{\psi}_{\ell,(i-1)T+t}.
 \end{align}

Since the combination matrix is left-stochastic, the algorithm does not converge to the minimizer of the average of the risk functions -- the solution of problem \eqref{eq:optProb}. Instead we can show that on average it will converge to the solution of the weighted average of the risk functions multiplied by the participation probabilities, where the weights are the entries of the Perron eigenvector of the mean combination matrix. Therfore, we let $\overline{p} = \col{\overline{p}_k}$ be the Perron eigenvector of $\mathbb{E} \bm{A}_{iT}$ and define:
\begin{align}\label{eq:optMod}
	w^o \eqdef \argmin \sum_{k=1}^K \overline{p}_k  q_kJ_k(w). 
\end{align}

\noindent{\bf Federated learning}. The following framework can be applied to the federated learning paradigm. The master-slave configuration can be viewed as a fully connected network, where every agent is a neighbour of every other agent. Additionally, the sampling size matches the number of participating agents, signifying that agents are not required to sample their neighbours. By framing the federated system within this interpretation, we can apply the findings discovered in this study. 
 
 In the original FedSGD algorithm \cite{mcmahan16}, each agent actively participates in every iteration, executing local update steps. Consequently, all probabilities are uniformly set to 1, $q_k = q_{\ell k} = 1$. As a result the step-size remains constant $\bm{\mu}_{k,(i-1)T+t} = \mu$. Additionally, the combination matrix alternates between the identity matrix and the full combination matrix with equally weighted entries:
 \begin{align}
 	\bm{A}_{(i-1)T+t} = \begin{cases}
 		\frac{1}{K}\mathds{1}^\tran \mathds{1}, & t = T \\
 		I, & t \neq T.
 	\end{cases}
 \end{align}  
 While, the FedAvg algorithm permits agent dropouts, resulting in non-unitary participation probabilities. The sampling probabilities $q_{ \ell k}$ are set to 1. Consequently, during the global iteration $i$, if $L_i$ agents participate, the combination weights at $t= T$ are given by:
 \begin{align}
 	\bm{a}_{ \ell k,iT} = \begin{cases}
	\frac{1}{ L_i }, &  \ell \text{ participating}  \\
	1, & k = \ell  \text{ not participating} \\
	0, & \text{otherwise}
\end{cases} 
 \end{align}
For $t\neq T$, the combination matrix remains identity. The step-size operates similary as in \eqref{eq:stepSize}. Thus in this particular setting, the combination matrix remains doubly-stochastic during each iteration. As such, the entries of the Perron eigenvector $\overline{p}_k$ of the mean combination matrix are equal to $1/K$, and the optimal model $w^o$ simplifies to the solution of the original optimization problem \eqref{eq:optProb}.

\subsection{Stability Analysis}
We begin by establishing the stability of the algorithm in the mean-square error sense. This involves starting with the formulation of the error recursion and then proceeding to define the gradient noise. Consequently, we define the error as $\we_{k,(i-1)T+t} = w^o - \w_{k,(i-1)T+t}$ and the gradient noise as:
\begin{align}\label{eq:gradNoise}
	\bm{s}_{k,(i-1)T+t} \eqdef & \widehat{\grad{w}J_k}(\w_{k,(i-1)T+t-1}) 
	\notag \\ &
	- {\grad{w}J_k}(\w_{k,(i-1)T+t-1}).
\end{align}
Then, by envoking the mean-value theorem \cite{sayed2014adaptation}, we can express the gradient as:
\begin{align}
	{\grad{w}J_k}(\w_{k,(i-1)T+t-1}) = &- \bm{H}_{k,(i-1)T+t-1}\we_{k,(i-1)T+t-1} \notag \\
	& - \grad{w}J_k(w^o), 
\end{align}
where we define the following terms:
\begin{align}
	\overline{\bm{H}}_{k,(i-1)T+t} &\eqdef \left( I - \bm{\mu}_{k , (i-1)T+t} \bm{H}_{k,(i-1)T+t-1} \right), 
	\\
	\bm{H}_{k,(i-1)T+t-1} &\eqdef \int_{0}^1 \grad{w}^ 2 J_k(w^ o - \tau \we_{k,(i-1)T+t-1}) d\tau .
\end{align}
Accordingly, the expression of the error recursion could be formulated as:
\begin{align}
	\we_{k,(i-1)T+t} = &\sum_{\ell \in \mathcal{N}_k} \bm{a}_{ \ell k, (i-1)T+t} \Big( \overline{\bm{H}}_{\ell,(i-1)T+t} \we_{\ell,(i-1)T+t-1}   
	\notag \\ &  
	+ \bm{\mu}_{\ell,(i-1)T+t} \left( \bm{s}_{\ell,(i-1)T+t}  - \grad{w}J_{\ell}(w^o) \right) \Big).
\end{align} 

Initially, we demonstrate that the stochastic gradient is an unbiased estimate of the true gradient and that the gradient noise has a finite second-order moment.
\begin{lemma}[\textbf{First and second-order moments of gradient noise}]
	The gradient noise defined in \eqref{eq:gradNoise} has zero-mean and bounded second-order moment, namely:
	\begin{align}
		\mathbb{E} \Vert \bm{s}_{k,(i-1)T+t} \Vert^2 \leq \beta_{s}^2 \mathbb{E}\Vert \we_{k,(i-1)T+t-1}\Vert^2 + \sigma_{s}^2,
	\end{align}
	where $\beta_{s}^2$ and $\sigma_{s}^2$ are some constants.
\end{lemma}
\begin{proof}
	Proof omitted due to space limitations.
\end{proof}
Following arguments similar to \cite{sayed2014adaptation}, it can be demonstrated that the algorithm achieves exponential convergence to a region around the true model. 
\begin{theorem}[\textbf{Mean-square stability}]\label{thrm:stab}
	Under assumptions \ref{assum:mat}, \ref{assum:fcts}, \ref{assum:Hess}, and for small enough step-size:
	\begin{align}
		\mu \leq \frac{2 \lambda_{\min}}{ \lambda^2_{\max} + \beta^2_{s}},
	\end{align}
	the individual errors converge exponentially fast:
	\begin{align}
		\limsup_{i \to \infty} \mathbb{E}\Vert \we_{k,(i-1)T+t} \Vert^ 2 \leq \frac{\sigma_{s}^2}{1-\gamma} \mu^2 = O(\mu),
	\end{align}
	where  $\gamma \eqdef \max\limits_k 1 - 2\mu q_k \lambda_{\min} + \mu^2 q_k (\lambda_{\max}^2 + \beta_{s}^2) \in [0,1)$ is the convergence rate.
\end{theorem}
\begin{proof}
	Proof omitted due to space limitations.
\end{proof}
As is evident from the theorem's statement, the rate and region of convergence are determined by the least active agent. To put it differently, the agent with the lowest participation probability $q_k$ slows down the overall algorithm. Moreover, given this agent's infrequent participation, it negatively impacts the overall network performance. Consequently, the strength of the network is contingent on its most fragile component.

In order to conduct the performance analysis, it is necessary to examine the fourth-order stability of the algorithm. Thus, by assuming that the gradient noise has bounded fourth-order moment, we can further expand upon the previous result in a manner similar to \cite{sayed2014adaptation}. 
\begin{theorem}[\textbf{Fourth-order stability}]\label{thrm:4thStab}
	If the gradient noise has a bounded fourth-order moment:
	\begin{align}
		\mathbb{E} \Vert \bm{s}_{k,(i-1)T+t}\Vert^4 \leq \beta_{s}^4 \mathbb{E} \Vert \we_{k,(i-1)T+t-1}\Vert^4 + \sigma_{s}^4,
	\end{align}
	and for small enough step-size, then:
	\begin{align}
		\limsup_{i\to \infty} \mathbb{E} \Vert \we_{k,(i-1)T+t}\Vert^4 \leq O(\mu^2).
	\end{align}
\end{theorem}
\begin{proof}
	Proof omitted due to space limitations.
\end{proof}

\subsection{Performance Analysis}
We are now ready to proceed with the performance analysis, wherein we will present an expression for the MSD for the federated setting. However, before delving into this, we lay out the following assumption concerning the noise process.

\begin{assumption}[\textbf{Noise process}]\label{assum:noiseProc}
	Define the covariance of the gradient noise:
	\begin{align}
		R_{k,(i-1)T + t}( w) &\eqdef \mathbb{E} \bm{s}_{k,(i-1)T +t}(w)\bm{s}^\tran_{k,(i-1)T +t}(w).
	\end{align}
	Then, for some positive constants $\kappa_s$ and $\alpha_s$, the covariance statisfies the following Lipschitz condition:
	\begin{align}
	&	\Vert \mathrm{diag}\{R_{k,(i-1)T + t}(w^o) - R_{k,(i-1)T + t}(\w_{k,(i-1)T +t-1})\}\Vert 
		\notag \\
	&	\leq \kappa_s \Vert \col{\we_{k,(i-1)T + t-1} }\Vert^{\alpha_s}, 
	\end{align} 
	and the following limit exists:
	\begin{align}
		R_k \eqdef \lim_{i\to \infty} R_{k,(i-1)T+t}(w^o).
	\end{align}
\qed
\end{assumption}

Using the aforementioned assumption and the smoothness assumption of the Hessians, we can derive an expression for the MSD. We introduce the matrix $\mathcal{G}_t$: 
\begin{align}
	 \mathcal{G}_t \eqdef &\mathbb{E} \left\{ \bm{\mathcal{A}}_{(i-1)T+t}^\tran \left( I- \bm{\mathcal{M}}_{k,(i-1)T+t}  \mathrm{diag}\{\grad{w}^2J_k(w^o)\} \right) \right.
	 \notag \\
	 & \left. \otimes_b \bm{\mathcal{A}}_{(i-1)T+t}^\tran \left( I- \bm{\mathcal{M}}_{k,(i-1)T+t}  \mathrm{diag}\{\grad{w}^2J_k(w^o)\} \right) \right\}
	 , 
\end{align}
where the operator $\otimes_b$ represents the block Kronecker product and:
\begin{align}
	\bm{\mathcal{A}}_{(i-1)T+t} &\eqdef \bm{A}_{(i-1)T+t} \otimes I, 
	\\
	\bm{\mathcal{M}}_{(i-1)T+t} &\eqdef 	\mathrm{diag}\{\bm{\mu}_{k,(i-1)T+t}\}.
\end{align}
For $t\neq T$, the matrix simplifies to:
\begin{align}
	\mathcal{G}_t  &= \mathrm{diag} \left\{ 
		G_{\ell k}
	\right\}, \\
	G_{\ell k} &\eqdef \begin{cases}
		(I-\mu q_{\ell}\grad{w}^2 J_{\ell}(w^o))  
		(I-\mu q_{k}\grad{w}^2 J_{k}(w^o)), & \ell \neq k \\
		I-q_k \mu \grad{w}^2 J_{k}(w^o), & \ell = k
	\end{cases}
\end{align}
At $t=T$, the matrix $\mathcal{G}_T$ captures the interdependencies among neighbours and the impact of neighbour sampling probabilities $q_{\ell k}$. Due to the complexity of the expression of $\mathcal{G}_T$, we omit its inclusion. 

We next introduce the matrix $\mathcal{C}_t$ which captures the dependency between the step-size and combination weights:
\begin{align}
	 \mathcal{C}_t \eqdef  &\mathbb{E}\left\{  (\bm{\mathcal{A}}_{(i-1)T+t}\otimes_b \bm{\mathcal{A}}_{(i-1)T+t})  \right.
	 \notag \\
	 &\times \left. (\bm{\mathcal{M}}_{(i-1)T+t}\otimes_b\bm{\mathcal{M}}_{(i-1)T+t}) \right\}.
\end{align}
For similar reasons, we refrain from explicitly formulating the matrix expression at $t=T$. Nonetheless, for $t\neq T$, we have:
\begin{align}
	\mathcal{C}_t &= \mathrm{diag} \left \{
		C_{\ell k}
 \right\}, \\
 	C_{\ell k} &\eqdef \begin{cases}
 		q_{k}q_{\ell}\mu^2, & \ell \neq k \\
 		q_k \mu^2,  & \ell = k
 	\end{cases}
\end{align}
Subsequently, we can formulate an expression for the MSD. 

\begin{theorem}[\textbf{Steady-state MSD}]\label{thrm:MSD}
	It holds that:
	\begin{align}\label{eq:MSDexp}
		\mathrm{MSD} = \frac{1}{K} z^\tran \mathrm{bvec}(I) + O(\mu^{1+ \alpha_0}),
	\end{align}
	where the block vectorization operator $\mathrm{bvec}$ stacks the columns of the blocks of the matrix, and:  
	\begin{align}
		\alpha_0 &\eqdef \frac{1}{2} \min \{ 1, \alpha_s\}, 
		\\
		z &\eqdef \left(I-\left(\mathcal{G}_T\mathcal{G}_t^{T-1} \right)^
	\tran\right)^{-1}  	\mathcal{C} \: \mathrm{bvec}(\mathrm{diag}\{R_k\} ),
  \\
  \mathcal{C} &\eqdef \left(I+\left(\mathcal{G}_T\mathcal{G}_t^{T-1}\right)^\tran \right) \mathcal{C}_T + \mathcal{G}_T^\tran\sum_{j=1}^{T-1}  \left(\mathcal{G}_t^{j-1}\right)^\tran \mathcal{C}_t .
	\end{align}
\end{theorem}
\begin{proof}
	Proof omitted due to space limitations. 
\end{proof}

\section{Experimental Results}\label{sec:exp}
We consider a linear regression problem of the form:
\begin{equation}
	\min_w \frac{1}{K}\sum_{k=1}^K \frac{1}{N} \sum_{n=1}^N \Vert \bm{d}_k(n) - \bm{u}_{k,n}^\tran w\Vert^2 .
\end{equation}
We generate for each agent a data set $\{\bm{d}_{k}(n),\bm{u}_{k,n}\}$ consisting of $N=10^6$ samples. These samples include five dimensional feature vectors $\bm{u}_{k,n}$ drawn from a normal distribution $\mathcal{N}(0,{R}_u)$ and an independent Guassian noise $\bm{v}_k(n) \sim \mathcal{N}(0,\sigma_{v,k}^2)$. A random generative model $w^{\star}$ is sampled from $\mathcal{N}(0,R_w)$ and the labels are determined by $\bm{d}_k(n) = \bm{u}_{k,n}^\tran w^{\star} + \bm{v}_{k}(n)$. For a linear regression problem, the optimal model $w^o$ can be calculated using $\widehat{R}_u$ and $\widehat{r}_{uv}$, which represent the sample covariance and cross-covariance:
\begin{align}
	w^o =   w^\star + \widehat{R}_u^{-1}\widehat{r}_{uv}.
\end{align}
The network comprises of $K = 20$ agents. We consider three different cases of the asynchronous network. Case 1 is the most general case where agent subsampling and local updates occur. The participation probabilities are set to $q_k = 0.5$. Neighbour sampling probabilities $q_{\ell k}$ are randomly assigned. Finally, the parameter $T$ is chosen as $100$. Case 2 only considers agent subsampling and no local updates. Thus, $T=1$ and the probabilities are kept as in case 1. Finally, case 3 assumes full agent participation with no subsampling of neighbourhoods ($q_k=q_{\ell k} =1$), while allows local updates ($T=100$). This case coincides with the FedSGD algorithm.  

Running the algorithm using a step-size $\mu = 0.0001$, we calculate the resultant MSD throughout the algorithm's progression. Afterward, we average the MSD across 5 experiments and illustrate the average curve allongside the theoretical MSD expression found in Theorem \ref{thrm:MSD}. Investigating Fig. \ref{fig:MSDplot}, we observe that as time passes the approximated MSD approaches the theoretical value, even for the general decentralized case.  Moreover, in the context of the given setup, whether local updates are used, as in case 1, or not, as in case 2, an identical behaviour is exhibited. Nonetheless, the primary factor affecting speedup becomes the agents' participation, particularly evident in case 3, where all agents participate. Consequently, the adoption of local updates has no discernible impact on the convergence rate of the algorithm. Furthermore, the three algorithms observe comparible theoretical MSD due to the small employed step-size.

\begin{figure}
	\includegraphics[width=0.5\textwidth]{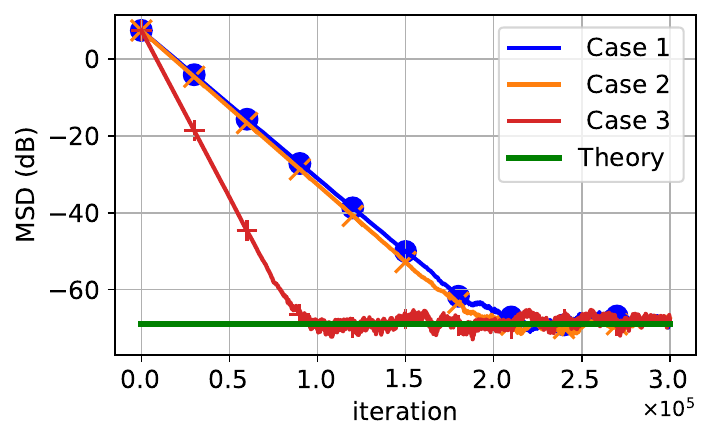}
	\caption{MSD curve for the asynchronous ATC diffusion algorithm.}\label{fig:MSDplot}
\end{figure}

\section{Conclusion}\label{sec:con}
In summary, this study focuses on asynchronous networks, where the primary goal is to solve a learning problem framed as an optimization task. Our investigation revolves around the premise that each agent autonomously decides when to participate in the algorithm and with which subset of its neighbourhood it will collaborate. Additionally, agents perform local updates before sharing their results. In this framework, we observe that in the federated setting the algorithm remains stable by continuously converging to an $O(\mu)$ neighbourhood of the optimal model $w^o$, just like the synchronous version of the algorithm. However, the rate and region of convergence are now influenced by the frequency of the nonparticipation of agents as well as the degree of connectedness of the network. Furthermore, these effects are aggregated in an actual expression of the MSD.

%


\vfill\pagebreak



\end{document}